%% file: tr.tex
\documentclass{article}

\usepackage{graphicx}
\usepackage{subfigure} 

\usepackage{algorithm}
\usepackage{algorithmic}

\usepackage{hyperref}

\usepackage{fullpage}

\usepackage {amsmath}
\usepackage {amsthm}
\usepackage{amssymb}

\def \theotherpaper {mcmahan10}

\begin{document} 

\title{Less Regret via Online Conditioning}

\author{Matthew Streeter\\ 
Google, Inc. \\
\texttt{\small mstreeter@google.com}
\and
H. Brendan McMahan \\
Google, Inc. \\
\texttt{\small mcmahan@google.com}
}

\maketitle

\def \yrcite {\cite}
\def \abovespace {}
\def \belowspace {}

\input {icml10_contents.tex}

\bibliographystyle{plain}
\bibliography{icml10}

\end{document}

%% file: icml10_contents.tex
\def \argmin {\mathop{\rm arg\,min}}

\newcommand{\grad}{\triangledown}
\newcommand{\norm}[1]{\ensuremath{\| #1 \|}}
\newcommand{\abs}[1]{\ensuremath{| #1 |}}
\newcommand{\paren}[1]{\ensuremath{\left(#1\right)}}
\newcommand{\Proj}{P}
\newcommand{\R}{\ensuremath{\mathbb{R}}}
\newcommand{\Regret}{\mathcal{R}}
\newcommand{\set}[1]{\ensuremath{\left\{#1\right\}}}
\newcommand{\tuple}[1]{\ensuremath{\left\langle#1\right\rangle}}

\newtheorem{lemma}{Lemma}
\newtheorem{theorem}{Theorem}
\newtheorem{corollary}{Corollary}

\begin{abstract} 
We analyze and evaluate an online gradient descent algorithm with
adaptive per-coordinate adjustment of learning rates.  Our algorithm
can be thought of as an online version of batch gradient descent with
a diagonal preconditioner.  This approach leads to regret bounds that
are stronger than those of standard online gradient descent
for general online convex optimization problems.
Experimentally, we show that our algorithm
is competitive with state-of-the-art algorithms for large
scale machine learning problems.
\end{abstract} 

\section{Introduction} \label{sec:intro}

In the past few years, online algorithms have emerged as
state-of-the-art techniques for solving large-scale machine learning
problems~\cite{bottou08,ma09identifying,zhang04}.  In addition to
their simplicity and generality, online algorithms are natural choices
for problems where new data is constantly arriving and rapid adaptation
is imporant.

Compared to the study of convex optimization in the batch (offline)
setting, the study of online convex optimization is relatively new.
In light of this, it is not surprising that performance-improving
techniques that are well known and widely used in the batch setting do
not yet have online analogues.  In particular, convergence rates in
the batch setting can often be dramatically improved through the use
of \emph {preconditioning}.  Yet, the online convex optimization
literature provides no comparable method for improving regret (the
online analogue of convergence rates).

A simple and effective form of preconditioning is to re-parameterize the loss
function so that its magnitude is the same in all coordinate
directions.  
Without this modification, a batch algorithm such as gradient descent
will tend to take excessively small steps along some axes and to
oscillate back and forth along others, slowing convergence.  In the
online setting, this rescaling cannot be done up front because the
loss functions vary over time and are not known in advance.  As a
result, when existing no-regret algorithms for online convex
optimization are applied to machine learning problems, they tend to
overfit the data with respect to certain features and underfit with
respect to others (we give a concrete example of this behavior in
\S\ref{sec:technicalmotivation}).

We show that this problem can be overcome in a principled way by using
online gradient descent\footnote{When loss functions are drawn IID, as
when online gradient descent is applied to a batch learning problem,
the term stochastic gradient descent is often used.}  with adaptive,
per-coordinate learning rates.  Our algorithm comes with worst-case
regret bounds (see Theorem~\ref{thm:better_regret}) that are never
worse than those of standard online gradient descent, and are much
better when the magnitude of the gradients varies greatly across
coordinates (this structure is common in large-scale problems of
practical interest).  Extending this approach, we give improved bounds
for generalized notions of strong convexity, bounds in terms of the
variance of cost functions, and bounds on adaptive regret (regret
against a drifting comparator).  Experimentally, we show that our
algorithm dramatically outperforms standard online gradient descent on
real-world problems, and is competitive with state-of-the-art
algorithms for online binary classification.

\subsection {Background and notation} \label {sec:background}

In an online convex optimization problem, we are given as input a
closed, convex feasible set $F$.  On each round $t$, we must pick a
point $x_t \in F$.  We then incur loss $f_t(x_t)$, where $f_t$ is a
convex function.  At the end of round $t$, the loss function $f_t$ is
revealed to us.  Our regret at the end of $T$ rounds is the difference
between our total loss and that of the best fixed $x \in F$ in
hindsight, that is
\[
\text{Regret} \equiv \sum_{t=1}^T f_t(x_t) -
\min_{x \in F} \set { \sum_{t=1}^T f_t(x) }.
\]

Sequential prediction using a generalized linear model is an important
special case of online convex optimization.  In this case, each $x_t \in \R^n$
is a vector of weights, where $x_{t, i}$ is the weight assigned to
feature $i$ on round $t$.  On round $t$, the algorithm makes a
prediction $p_t(x_t) = \ell(x_t \cdot \theta_t)$, where $\theta_t \in \R^n$ is
a feature vector and $\ell$ is a fixed link function (e.g.,
$\ell(\alpha) = \frac {1} {1 + \exp(-\alpha)}$ for logistic
regression, $\ell(\alpha) = \alpha$ for linear regression).  The
algorithm then incurs loss that is some function of the prediction
$p_t$ and the label $y_t \in \R$ of the example.  For example, in
logistic regression the loss is $f_t(x) = y_t \log p_t(x) + (1 - y_t)
\log (1 - p_t(x))$, and in least squares linear regression the loss is
$f_t(x) = (y_t - p_t(x))^2$.  In both of these examples, it can be
shown that $f_t$ is a convex function of $x$.

We are particularly interested in online gradient descent and
generalizations thereof.  
Online gradient descent chooses $x_1$ arbitrarily, and thereafter
plays
\begin{equation}\label{eq:ogd}
x_{t+1} = \Proj(x_t - \eta_t g_t)
\end{equation}
where $\eta_1$, $\eta_2$, \ldots, $\eta_T$ is a sequence of learning
rates, $g_t \in \grad f_t(x_t)$ is a subgradient of $f_t(x_t)$, and
$\Proj(x) = \argmin_{y \in F} \set { \norm {x - y}}$ is the projection
operator, where $\norm{\cdot}$ is the L2 norm.  When the learning
rates are chosen appropriately, online gradient descent obtains regret
$O(G D \sqrt T)$, where $D = \max_{x, y \in F}
\set { \norm {x - y} }$ is the diameter of the feasible set and $G =
\max_{t} \set { \norm{g_t} }$ is the maximum norm of the gradients.
Thus, as $T
\rightarrow \infty$, the average loss of the points $x_1, x_2, \ldots,
x_T$ selected by online gradient descent is as good as that of any
fixed point $x \in F$ in the feasible set.  It is perhaps surprising
that this performance guarantee holds for \emph {any} sequence of loss
functions, and in particular that the bounds holds even if the
sequence is chosen adversarially.

\section {Motivations} \label {sec:technicalmotivation}

It is well-known that batch gradient descent performs poorly in the
presence of so-called \emph {ravines}, surfaces that curve more
steeply in some directions than in others \cite{sutton86}.  In this
section we give examples showing that when the slope of the loss
function or the size of the feasible set varies widely across
coordinates, gradient descent incurs high regret in the online
setting.  These observations motivate the use of per-coordinate
learning rates (which can be thought of as an adaptive diagonal preconditioner).

\subsection {A motivating application}

Consider the problem of trying to predict the probability that a user
will click on an ad when it is shown alongside
search results for a particular query, using a generalized linear
model.
For simplicity, imagine there is only one ad, and we wish to predict
its click-through rate on many different queries.  On a large search
engine, a popular query will occur orders of magnitude more often than
a rare query.  For queries that occur rarely, it is necessary to use a
relatively large learning rate in order for the associated feature
weights to move significantly away from zero.  But for popular
queries, the use of such a large learning rate will cause the feature
weights to oscillate wildly, and so the predictions made by the
algorithm will be unstable.  Thus, gradient descent with a global
learning rate cannot simultaneously perform well on common queries and
on rare ones.  Because rare queries are more numerous than common
ones, performing poorly on either category leads to substantial
regret.

\subsection {Tradeoffs in one dimension} \label {sec:tradeoffs}

We first consider gradient descent in one dimension, with a fixed
learning rate $\eta$ (later we generalize to arbitrary
non-increasing sequences of learning rates).

If $\eta$ is too large, the algorithm may oscillate about
the optimal point and thereby incur high regret.  As a simple example,
suppose the feasible set is $[0, D]$, and the loss function on
each round is $f_t(x) = G \abs{x - \epsilon}$, for some small positive
$\epsilon$.  Then $\grad f_t(x) = -G$ if $x < \epsilon$ and $\grad
f_t(x) = G$ if $x > \epsilon$.  It is easy to verify that if the
algorithm plays $x_1 = 0$ initially, it will play $x_t = 0$ on odd
rounds and $x_t = G \eta$ on even rounds, assuming $\epsilon < G \eta
\le D$.  Thus, after $T$ rounds the algorithm incurs total loss
$\frac T 2 G \epsilon + \frac T 2 G (G \eta - \epsilon) = \frac T 2
G^2 \eta$.  Always playing $x = \epsilon$ would incur zero loss, so
the regret is $\frac T 2 G^2 \eta$.

On the other hand, if $\eta$ is too small then $x_t$ may stay close to
zero long after the data indicates that a larger $x$ would
incur smaller loss.  For example, suppose $f_t(x) = -G x$ always.  Then
$x_t = \min \set { D, (t - 1) G \eta}$.  For the first $\frac {D} {2 G
\eta}$ rounds, $x_t \le \frac D 2$ and therefore our per-round regret
relative to the comparator $x = D$ is at least $\frac {G D} {2}$ on
these rounds.  Thus, overall regret is at least $\frac {G D} {2} \min
\set {T, \frac {D} {2 G \eta} } = \frac {D^2} {4 \eta}$,
assuming that $\frac {D} {2 G \eta} \le T$.
Thus, for any choice of $\eta$ there exists a problem where
\[
\max \set { \frac {D^2} {4 \eta}, G^2 \eta \frac T 2 } 
\leq \text{Regret}
\leq \frac {D^2} {2 \eta} + G^2 \eta \frac T 2, 
\]
where the upper bound is adapted from Zinkevich \yrcite{zinkevich03}.
Thus, by setting $\eta = \frac {D} {G \sqrt T}$ (which minimizes the
upper bound) we minimize worst-case regret up to a constant factor.
Note that this choice of $\eta$ satisfies the constraints $\frac {D}
{2 T} \le G \eta \le D$, as was assumed earlier.

The fact that the optimal choice of $\eta$ is proportional to $\frac D
G$ captures a fundamental tradeoff.  When the feasible set is large
and the gradients are small, we must use a larger learning rate in
order to be competitive with points in the far extremes of the
feasible set.  On the other hand, when the feasible set is small and
the gradients are large, we must use a smaller learning rate in order
to avoid the possibility of oscillating between the extremes and
performing poorly relative to points in the center.

Because the relevant values of $D$ and $G$ will in general be
different for different coordinates, a gradient descent algorithm that
uses the same learning rate for all coordinates is doomed to either
underfit on some coordinates or oscillate on others.  To handle this,
we must use different learning rates for different coordinates.
Furthermore, because the magnitude $G$ of the gradients is not known
in advance and can change over time, we must incorporate it into our
choice of learning rate in an online fashion.

\subsection {A bad example for global learning rates} \label {sec:badexample}

We now exhibit a class of online convex optimization problems
where the use of a coordinate-independent learning rate forces regret to
grow at an asymptotically larger rate than with a per-coordinate learning rate.
This result is summarized in the following theorem.

\begin{theorem} \label {thm:bad_class}
There exists a family of online convex optimization problems,
parameterized by their lengths (number of rounds $T$), where
gradient descent with a non-increasing global learning rate incurs
regret at least $\Omega(T^\frac{2}{3})$, whereas gradient descent with
an appropriate per-coordinate learning rate has regret $O(\sqrt{T})$.
\end{theorem}

The $\Omega(T^{\frac 2 3})$ lower bound stated in Theorem \ref
{thm:bad_class} does not contradict the previously-stated $O(G D \sqrt
T)$ upper bound on the regret of online gradient descent, because in
this family of problems $D = T^{\frac 1 6}$ (and $G = 1$).

\begin {proof} [Proof of Theorem \ref {thm:bad_class}]
To prove this theorem, we interleave instances of the
two classes of one-dimensional subproblem discussed in
\S\ref{sec:tradeoffs}, setting $G = 1$ and setting the feasible set to
$[0, 1]$.  We have one subproblem of the first type, lasting for $T_0$
rounds, followed by $C$ subproblems of the second type, each lasting
$T_1$ rounds.  Each subproblem is assigned its own coordinate.
Formally, the loss function is
\[
f_t(x_t) = \left \{
\begin {array} {l l}
\abs{x_{t,1} - \epsilon} & \mbox { if } t \le T_0 \\
-x_{t, j} & \mbox { if } t > T_0
\mbox { where $j = 1 + \left \lceil \frac {T - T_0} {T_1} \right \rceil$}  \\
\end {array}
\right .
\]

On each round, only one component of the gradient vector is non-zero.
Thus, running gradient descent with global learning rate $\eta$ is
equivalent to running a separate copy of gradient descent on each subproblem,
where each copy uses learning rate $\eta$.  Moreover, overall regret
is simply the sum of the regret on each subproblem.
Thus, by the lower bounds stated \S\ref{sec:tradeoffs}, regret is at
least
\[
\frac {T_0} {2} \eta + \frac C 2 \min \set { T_1, \frac {1} {2 \eta}}
\]
(note that $G = D = 1$).

If we set $C = T_1 = T_0^{\frac 1 3}$, this expression is
$\Omega(T^{\frac 2 3})$.  To see this, first note that if $T_1 \le
\frac {1} {2 \eta}$ then the second term is already $\Omega(T_0^{\frac
  2 3}) = \Omega(T^{\frac 2 3})$
(note that $T = T_0 + T_0^{\frac{2}{3}} \leq 2 T_0$).
Otherwise, a simple minimization over $\eta$ shows that the sum is
$\Omega(T_0^{\frac 2 3})$.  Because regret on the first subproblem is
an increasing function of $\eta$, and regret on all later subproblems
is a decreasing function of $\eta$, the same $\Omega(T^{\frac 2 3})$
lower bound holds for any non-increasing sequence $\eta_1, \eta_2,
\ldots, \eta_T$ of per-round learning rates.  Thus, we have proved the
first part of the theorem.

Now consider the alternative of letting the learning rate for each
coordinate vary independently.  On a one-dimensional subproblem with
feasible set $[0, 1]$ and gradients of magnitude at most 1, gradient
descent using learning rate $\frac {1} {\sqrt s}$ on round $s$
of the subproblem obtains regret $O(\sqrt S)$ on a subproblem of
length $S$ \cite {zinkevich03}.  Thus, if we ran an independent copy
of this algorithm on each coordinate, we would obtain regret $O( \sqrt
{T_0} + C \sqrt {T_1} ) = O(\sqrt {T_0}) = O(\sqrt T)$, which
completes the proof.
\end {proof}

\section{Improved Regret Bounds using Per-Coordinate Learning Rates}
\label{sec:per_coord_rate}

Zinkevich~\yrcite{zinkevich03} proved bounds on the regret of online
gradient descent (which chooses $x_t$ according to Equation~\eqref{eq:ogd}).
Building on his analysis,
we improve these bounds by adjusting the learning rates on a
per-coordinate basis.  Specifically, we obtain these bounds by 
constructing the vector $y_t$ by 
\begin{equation}\label{eq:percoord-gd}
	y_{t, i} = x_{t, i} - g_{t, i} \eta_{t, i}
\end{equation}
where $\eta_t$ is a vector of learning rates, one for each coordinate.
We then play $x_t = P(y_t)$.  We prove bounds for feasible sets
defined by axis-aligned constraints, $F =
\times_{i=1}^n [a_i, b_i]$.  Many machine learning problems can be
solved using feasible sets of this form, as our experiments
demonstrate.\footnote{Our techniques can be extended to arbitrary
feasible sets using a somewhat different algorithm, but the proofs are
signicantly more technical \cite{\theotherpaper}.}

\subsection{A better global learning rate} \label {subsec:better_global_rate}

We first give an improved regret bound for gradient descent with a global
(coordinate-independent) learning rate.  In the next subsection, we
make use of this improved bound in order to prove the desired
bounds on the regret of gradient descent with a per-coordinate
learning rate.

Zinkevich~\yrcite {zinkevich03} showed that if we run gradient descent
with a non-increasing sequence $\eta_1, \eta_2, \ldots, \eta_T$ of
learning rates, regret is bounded by
\begin {equation} \label {eq:regret}
B(\eta_1, \eta_2, \ldots, \eta_T) =
D^2 \frac {1} {2 \eta_T} + \frac 1 2 \sum_{t=1}^T \norm {g_t}^2 \eta_t .
\end {equation}

To guard against the worst case, it is natural to choose our sequence
of learning rates so as to minimize this bound.  Doing so is
problematic, however, because in the online setting the gradients
$g_1, g_2, \ldots, g_T$ are not known in advance.  Perhaps
surprisingly, we can come within a factor of $\sqrt 2$ of the optimal
bound even without having this information up front, as the following
theorem shows.

\newcommand{\Rmin}{R_{\rm{min}}}
\begin {theorem} \label {thm:global_rate}
Setting $\eta_t = \frac {D} { \sqrt { 2 \sum_{s=1}^t \norm {g_s}^2 }}$
yields regret $D \sqrt{ 2 \sum_{t=1}^T \norm {g_t}^2} = \sqrt 2 \cdot
\Rmin$, where $\Rmin = \min_{\eta_1, \eta_2, \ldots, \eta_T:\ \eta_1
  \ge \eta_2 \ge \ldots \ge \eta_T} \set { B(\eta_1, \eta_2, \ldots,
  \eta_T) }$.
\end {theorem}
\begin {proof}
Plugging the formula for $\eta_t$ into \eqref {eq:regret}, and
then using Lemma \ref {lem:sum} (below), we see that regret
is bounded by
\begin {align*}
  \frac{1}{2} D \paren{ \sqrt {2 \sum_{t=1}^T \norm {g_t}^2 }
    + \sum_{t=1}^T \frac {\norm{g_t}^2} {\sqrt { 2 \sum_{s=1}^t \norm
        { g_s }^2 } } } \\ \le D \sqrt{ 2 \sum_{t=1}^T \norm
    {g_t}^2} \mbox { .}
\end {align*}
We now compute $\Rmin$.
First, note that if $\eta_t > \eta_{t+1}$ for some $t$ then we could
reduce the second term in $B(\{\eta_t\})$ by making $\eta_t$ smaller.
Because the sequence is constrained to be non-increasing, it follows
that the bound is minimized using a constant learning rate $\eta$.
A simple minimization then shows that it is optimal to set
$\eta = \frac {D} { \sqrt { \sum_{t=1}^T \norm {g_t}^2 }}$.
which gives regret $D \sqrt { \sum_{t=1}^T \norm {g_t}^2 } = \Rmin$.
\end {proof}

A related result appears in \cite {bartlett08}, giving improved bounds
in the case of strongly convex functions but worse constants than ours
in the case of linear functions.

\begin {lemma} \label {lem:sum}
For any non-negative real numbers $x_1, x_2, \ldots, x_n$,
\[
\sum_{i=1}^n \frac { x_i } { \sqrt { \sum_{j=1}^i x_j } } \le 2 \sqrt
    { \sum_{i=1}^n x_i } \mbox { .}
\]
\end {lemma}
\begin {proof}
The lemma is clearly true for $n = 1$.  Fix some $n$, and assume the
lemma holds for $n - 1$.  Thus,
\begin {align*}
	\sum_{i=1}^n \frac { x_i } { \sqrt { \sum_{j=1}^i x_j } }
& \le  2 \sqrt { \sum_{i=1}^{n-1} x_i }
  + \frac { x_n } { \sqrt { \sum_{i=1}^n x_i } } \\
& = 2 \sqrt {Z - x} + \frac {x} { \sqrt { Z } }
\end {align*}
where we define $Z = \sum_{i=1}^{n} x_i$ and $x = x_n$.  The
derivative of the right hand side with respect to $x$ is $\frac {-1}
{\sqrt {Z -x}} + \frac {1} {\sqrt Z}$, which is negative for $x > 0$.
Thus, subject to the constraint $x \ge 0$,
the right hand side is maximized at $x = 0$, and is
therefore at most $2 \sqrt Z$.
\end {proof}

\subsection {A per-coordinate learning rate} \label {subsec:per_coord_rate}

\begin{algorithm}[tb]
   \caption{Per-coordinate gradient descent} \label {alg:percoordinate}
\begin{algorithmic}
   \STATE {\bfseries Input:} feasible set $F = \times_{i=1}^n [a_i, b_i]$
   \STATE Initialize $x_1 = 0$ and  $D_i = b_i - a_i$.
   \FOR{$t=1$ {\bfseries to} $T$}
   \STATE Play the point $x_t$.
   \STATE Receive loss function $f_t$, set $g_t = \grad f_t(x_t)$.
   \STATE Let $y_{t+1}$ be a vector whose $i^{th}$ component is
   $y_{t+1, i} = x_{t, i} - \eta_{t, i} g_{t, i}$, where
   $\eta_{t, i} = \frac {D_i} {\sqrt {\sum_{s=1}^t g^2_{s, i}}}$.
   \STATE Set $x_{t+1} = \Proj(y_{t+1})$.
   \ENDFOR
\end{algorithmic}
\end{algorithm}

We can improve the above bound by running, for each coordinate, a
separate copy of gradient descent that uses the learning rate given in
the previous section (see Algorithm \ref {alg:percoordinate}).
Specifically, we use the update of Equation~\eqref{eq:percoord-gd}
with $\eta_{t, i} = \frac {D_i} {\sqrt{\sum_{s=1}^t g_{s, i}^2 }}$,
where $D_i = b_i - a_i$ is the diameter of the feasible set along
coordinate $i$.

The following theorem makes three important points about the
performance of Algorithm~\ref{alg:percoordinate}: (\emph{i}), its
regret is bounded by a sum of per-coordinate bounds, each of the same
form as \eqref{eq:regret}; (\emph{ii}) the algorithm's choice
of $\eta_{t, i}$ gives a regret bound that is only a factor of $\sqrt 2$
worse than if the bound had been optimized knowing $g_1, g_2, \dots,
g_T$ in advance; and,
(\emph{iii}), the regret bound of
Algorithm~\ref{alg:percoordinate} is never worse than
the bound for global learning rates stated in
Theorem~\ref{thm:global_rate}.  Futhermore, as illustrated in
Theorem~\ref{thm:bad_class}, the per-coordinate bound can be better by
an arbitrarily large factor if the magnitude of the gradients varies
widely across coordinates.

\begin {theorem} \label {thm:better_regret}
Let $F = \times_{i=1}^n [a_i, b_i]$.  Then,
Algorithm~\ref {alg:percoordinate} has regret bounded by
$\sum_{i=1}^n B_i(\set{\eta_{t, i}})$, where
\[
B_i(\set{\eta_{t, i}}) \equiv 
D_i^2 \frac {1} {2 \eta_{T, i}} + \frac 1 2 \sum_{t=1}^T g_{t, i}^2 \eta_{t, i}
\mbox { .}
\]
Setting $\eta_{t, i} = \frac {D_i} {\sqrt{\sum_{s=1}^t g_{s, i}^2 }}$,
the bound becomes
\begin{equation}\label{eq:percoordbound}
\sum_{i=1}^n D_i \sqrt { 2 \sum_{t=1}^T g_{t, i}^2 }
= \sqrt 2 \sum_{i=1}^n \Rmin^i
\end{equation}
where $\Rmin^i = \min_{\{\eta_{t, i}\}: \eta_{1, i} \ge \eta_{2, i}
\ge \ldots \ge \eta_{T, i}} \set { B_i(\{\eta_{t, i}\}) }$.  This is a
stronger guarantee than Theorem~\ref{thm:global_rate}, in
that
\begin {equation} \label {eq:better_regret}
\sum_{i=1}^n D_i \sqrt { 2 \sum_{t=1}^T g_{t, i}^2 } \le
D \sqrt{ 2 \sum_{t=1}^T \norm {g_t}^2}
\end {equation}
where $D = \sqrt{\sum_{i=1}^n D_i^2}$ is the diameter of the set $F$.
\end {theorem}
\begin {proof}
Zinkevich\yrcite {zinkevich03} showed that, so long as our algorithm
only makes use of $\grad f_t(x_t)$, we may assume without loss of
generality that $f_t$ is linear, and therefore $f_t(x) = g_t \cdot x$
for all $x \in F$.
If $F$ is a hypercube, then the projection operator $P(x)$ simply projects
each coordinate $x_i$ indepdently onto the interval $[a_i, b_i]$.
Thus, in this special case, we can think of each coordinate $i$
as solving a separate online convex optimization problem where the
loss function on round $t$ is $g_{t, i} \cdot x$.  Thus,
Equation~\eqref{eq:regret} implies that for each $i$,
\[
\sum_{t=1}^T g_{t, i} x_{t, i} - \min_{y \in [a_i, b_i]} \set {
  \sum_{t=1}^T g_{t, i} y} \le B_i(\set{\eta_{t, i}}) \mbox{ .}
\]
Summing this bound over all $i$, we get the regret bound
\begin {equation}
	\sum_{t=1}^T g_t \cdot x_t - \min_{x \in F} \set {
          \sum_{t=1}^T g_t \cdot x} \le \sum_{i=1}^n B_i(\set{\eta_{t, i}}) .
\end {equation}
Applying Theorem~\ref{thm:global_rate} to each one-dimensional
problem, we get $B_i(\set{\eta_{t, i}}) =
D_i \sqrt { 2 \sum_{t=1}^T g_{t, i}^2 } = \sqrt 2 \cdot \Rmin^i$ $\forall i$.

To prove inequality \eqref {eq:better_regret}, let $\vec D \in \R^n$
be a vector whose $i^{th}$ component is $D_i$, and let $\vec g \in
\R^n$ be a vector whose $i^{th}$ component is $\sqrt { 2 \sum_{t=1}^T
  g_{t, i}^2 }$, so the left-hand side of \eqref{eq:better_regret} can be
  written as $\vec D \cdot \vec g $. Then, 
using the Cauchy-Schwarz inequality,
\[
 \vec D \cdot \vec g 
 \le \norm{\vec D} \cdot \norm{\vec g} 
= \sqrt {\sum_{i=1}^n D_i^2} \sqrt { 2 \sum_{i=1}^n \sum_{t=1}^T g_{t, i}^2 }
\mbox { .}
\]
The right hand side simplifies to
$D \sqrt { 2 \sum_{t=1}^T \norm{g_t}^2 }$.
\end {proof}

\section{Additional Improved Regret Bounds} \label{sec:morebounds}

The approach of bounding overall regret in terms of the sum
of regret on a set of one-dimensional problems can be used to obtain
additional regret bounds that improve over those of previous work, in
the special case where the feasible set is a hypercube.  The key
observation is captured in the following lemma.

\begin {lemma} \label {lem:decomposition}
Consider an online optimization problem with feasible set 
$F = \times_{i=1}^n [a_i, b_i]$ 
and loss functions $f_1, f_2,
\ldots, f_T$.  For each $t$, let $\ell_t(x) = \sum_{i=1}^n \ell_{t,
i}(x_i)$ be a lower bound on $f_t$ (i.e., $f_t(x) \ge \ell_t(x)$ for
all $x \in F$).  Further suppose that $f_t(x_t) = \ell_t(x_t)$ for all
$t$, where $\set{x_t}$ is the sequence of points played by an online
algorithm.  Consider the composite online algorithm formed by running
a 1-dimensional algorithm independently for each coordinate $i$ on
feasible set $[a_i, b_i] \subseteq \R^n$, with loss function
$\ell_{t,i}$ on round $t$.  Let
\[
R = \sum_{t=1}^T f_t(x_t) - \min_{x \in F} \set { \sum_{t=1}^T f_t(x) }
\]
be the total regret of the composite algorithm, and let
\[
R_i = \sum_{t=1}^T \ell_{t, i}(x_{t, i}) -
  \min_{x_i \in [a_i, b_i]} \set { \sum_{t=1}^T \ell_{t, i}(x_i)  }
\]
be the regret incurred by the algorithm responsible for choosing the
$i^{th}$ coordinate.  Then $R \le \sum_{i=1}^n R_i$.
\end {lemma}
\begin {proof}
Because $f_t(x) \ge \ell_t(x)$ $\forall x$, and $f_t(x_t) = \ell_t(x_t)$,
\begin {align*}
R 
& \le \sum_{t=1}^T \ell_t(x_t) - \min_{x \in F} \set { \sum_{t=1}^T \ell_t(x) } \\
& = \sum_{t=1}^T \ell_t(x_t) - \sum_{i=1}^n \min_{x_i \in [a_i, b_i]}
\set { \sum_{t=1}^T \ell_{t, i}(x_i) }
& = \sum_{i=1}^n R_i \mbox { .} 
\end {align*} 
\end {proof}

Importantly, for arbitrary convex functions, we can always construct
such independent lower bounds by choosing
$ \ell_t(x) = f_t(x) + \grad f(x_t)(x - x_t)$,
as long as we add a ``bias'' coordinate where $a_i = b_i = 1$.  A
similar observation was originally used by
Zinkevich~\yrcite{zinkevich03} to show that any algorithm for online
linear optimization can be used for online convex optimization.  We
used this fact in the proof of Theorem~\ref{thm:better_regret},
where we only analyzed the linear case.

This simple lemma has powerful ramifications.  We now discuss several
improved guarantees that can be obtained by applying it to known
online algorithms.  For simplicity, when stating these bounds we
assume that the feasible set is $F = [0, 1]^n$ and that the gradients
of the loss functions are componentwise upper bounded by 1 (that is,
$|(\grad f_t(x_t))_i| \le 1$ for all $t$ and $i$).

\subsection {More general notions of strong convexity}

A function $f$ is $H$-strongly convex if, for all $x, y
\in F$, it holds that $f(y) \ge f(x) + \grad f(x) \cdot (y - x) + \frac H 2
\norm{y - x}^2$.  Strongly convex functions arise, for example, when
solving learning problems subject to L2 regularization.

Bartlett et al.~\yrcite{bartlett08} give an online convex optimization
algorithm whose regret is
\[
O\paren{n \cdot \min\set{ \sqrt T, \frac {1} {H} \log T }}
\]
where $H$ is the largest constant such that each $f_t$ is $H$-strongly
convex.  We can generalize the concept of strong convexity as follows.
We say that $f$ is strongly convex with respect to the vector $\vec H$
if, for all $x, y \in F$, $f(y) \ge f(x) + \grad f(x) \cdot (y - x) +
\sum_{i=1}^n \frac {\vec H_i} {2} (y_i - x_i)^2$.  Suppose we run the
algorithm of Bartlett et al.\ independently for each coordinate,
feeding back $\ell_{t, i}(y_i) = \frac 1 n f_t(x_t) + \grad
f_t(x_t)_i \cdot (y_i - x_{t, i}) + \frac {\vec H_i} {2} (y_i - x_{t, i})^2$
to the algorithm responsible for choosing coordinate $i$ (we can
always choose $\vec{H}_i \geq H$).  Applying
Lemma~\ref{lem:decomposition}, we obtain a regret bound
\[
O\paren{\sum_{i=1}^n \min \set { \sqrt T, \frac {1} {\vec H_i} \log T }}
\mbox { .}
\]
This bound is never worse than
the previous one, and is better if the degree of strong convexity differs
substantially across different coordinates (e.g., if using different 
L2 regularization parameters for different classes of features).

\subsection {Tighter bounds in terms of variance}

Hazan and Kale \yrcite{hazan08} give a bound on gradient descent's
regret in terms of the variance of the sequence of gradients.
Specifically, their algorithm has regret $O(\sqrt {n V})$, where $V =
\sum_{t=1}^T \norm{g_t - \mu}^2$ and $\mu = \frac 1 T \sum_{t=1}^T g_t$,
where $g_t = \grad f_t(x_t)$.

By running a separate copy of their algorithm on each coordinate, we
can instead obtain a bound of $O(\sum_{i=1}^n \sqrt {V_i})$, where
$V_i = \sum_{t=1}^T (g_{t, i} - \mu_i)^2$.

To compare the bounds, let $\vec v \in \R^n$ be a
vector whose $i^{th}$ component is $\sqrt {V_i}$, and let $\vec 1 \in
\R^n$ be a vector whose components are all 1.  Note that $\norm{\vec
  v} = \sqrt {\sum_{i=1}^n V_i} = \sqrt {V}$.  Using the
Cauchy-Schwarz inequality,
\[
\sum_{i=1}^n \sqrt {V_i} = \vec 1 \cdot \vec v
\le \norm{\vec 1} \cdot \norm{\vec v} = \sqrt {n V} \mbox { .}
\]
Thus, the bound obtained by running separate copies of the algorithm for
each coordinate is never worse than the original bound, and is
substantially better when the variance $V_i$ varies greatly across
coordinates.

\subsection {Adaptive regret}

One weakness of standard regret bounds like those stated so far is
that they bound performance only in terms of the static optimal
solution over all $T$ rounds.  In a non-stationary environment, it is
desirable to obtain stronger guarantees.  For example, suppose the
feasible set is $[0, 1]$, $f_t(x) = x$ for the first $\frac T 2$
rounds and $f_t(x) = -x$ thereafter.  Then an algorithm that plays $x_t
= 0$ for all $t$ has 0 regret, yet its loss on the final $\frac T 2$
rounds is $\frac T 2$ worse than if it had played the point $x = 1$
for those rounds.  Indeed, standard regret-minimizing algorithms fail
to adapt in simple examples such as this.

Hazan and Seshadhri~\yrcite{hazan09} define \emph {adaptive regret} as
the maximum, over all intervals $[T_0, T_1]$, of the regret
$\sum_{t=T_0}^{T_1} f_t(x_t) - \min_{x \in F} \set {
  \sum_{t=T_0}^{T_1} f_t(x) } $ incurred over that interval.  For
$H$-strongly convex functions, their algorithm achieves adaptive
regret $O\paren{\frac {1} {H} \log^2 T}$.

By running an independent copy of their algorithm on each coordinate,
we can obtain the following guarantee.  Consider an
arbitrary sequence $Z = \tuple{z_1, z_2, \ldots, z_T}$ of points in
$F$, and let $R_Z = \sum_{t=1}^T f_t(x_t) - f_t(z_t)$ be the regret
relative to that sequence.  Holding
$H$ constant for simplicity, the adaptive regret bound just stated
implies that the algorithm of Hazan and Seshadhri~\yrcite{hazan09}
obtains $R_Z = O((N+1) \log^2 T)$, where $N$ is the number of values
of $t$ for which $z_t \neq z_{t+1}$ (this follows by summing adaptive
regret over the $N + 1$ intervals where $z_t$ is constant).  Using
separate copies for each coordinate, we instead obtain
\[
R_Z = O\paren{\sum_{i=1}^n (N_i + 1)  \log^2 T}
\]
where $N_i$ is the number of values of $t$ where $z_{t, i} \neq z_{t+1, i}$.
This bound is never worse than the previous one, and is better
when some coordinates of the vectors in $Z$ change more frequently than others.

This provides an improved performance
guarantee when the environment is stationary with respect to some
coordinates and non-stationary with respect to others.  This could
happen, for example, if the effect of certain features 
(e.g., features for advertisers in certain business sectors) 
changes over time, but the effect of other features remains constant.

\section{Experimental Evaluation} \label{sec:exp}

In this section, we evaluate gradient descent with per-coordinate
learning rates experimentally on several machine learning problems.

\subsection {Online binary classification}

\begin{table}[t]
\caption{
Hinge loss and accuracy in the online setting on binary classification
problems.}
\label{tab:classification}
\begin{center}
\begin{small}
\begin{sc}
\begin{tabular}{lcccccr}
\hline
\abovespace\belowspace
          Data &         Global &      Per-Coord &             CW &             PA \\
\hline
\noalign{\smallskip}
\multicolumn{5}{l}{\textbf{Hinge loss}} \\
         books &          0.606 & \textbf{0.545} &          0.871 &          0.672 \\
           dvd &          0.576 & \textbf{0.529} &          0.851 &          0.637 \\
   electronics &          0.509 & \textbf{0.452} &          0.802 &          0.555 \\
       kitchen &          0.470 & \textbf{0.419} &          0.787 &          0.520 \\
          news &          0.171 & \textbf{0.140} &          0.512 &          0.245 \\
          rcv1 &          0.076 & \textbf{0.070} &          0.542 &          0.094 \\
\hline
\noalign{\smallskip}
\multicolumn{5}{l}{\textbf{Fraction of mistakes}} \\
         books &          0.259 & \textbf{0.211} &          0.215 &          0.254 \\
           dvd &          0.238 &          0.208 & \textbf{0.203} &          0.240 \\
   electronics &          0.209 & \textbf{0.175} &          0.177 &          0.194 \\
       kitchen &          0.180 & \textbf{0.151} &          0.153 &          0.175 \\
          news &          0.064 & \textbf{0.050} &          0.054 &          0.060 \\
          rcv1 &          0.027 & \textbf{0.025} &          0.039 &          0.034 \\
\hline
\end{tabular}
\end{sc}
\end{small}
\end{center}
\vskip -0.2in
\end{table}

We first compare the performance of online gradient descent with that of
two recent algorithms for text classification: the Passive-Aggressive
(PA) algorithm~\cite{crammer06passive}, and confidence-weighted (CW)
linear classification~\cite{drezde08}.  The latter algorithm has been
demonstrated to have state-of-the-art performance on large real-world
problems~\cite{ma09identifying}.

We used four sentiment classification data sets (Books, Dvd,
Electronics, and Kitchen), available from~\cite{sentiment}, each with
1000 positive examples and 1000 negative examples,\footnote{We used
the features provided in processed\_acl.tar.gz, and scaled each vector
of counts to unit length.  } as well as the scaled versions of the
rcv1.binary (677,399 examples) and news20.binary (19,996 examples)
data sets from LIBSVM~\yrcite{libsvmdata}. For each data set, we shuffled the
examples and then ran each algorithm for one pass over the data,
computing the loss on each event before training on it.

For the online gradient descent algorithms, we set $F = [-R, R]^n$ for
$R=100$.  We found that the learning rate suggested by
Theorem~\ref{thm:better_regret} was too aggressive in practice when the feasible
set is large (note that it moves a feature's weight to the maximum
value the first time it sees a non-zero gradient for that feature).
In order to improve performance, we did some parameter tuning.
For Algorithm~\ref{alg:percoordinate} (Per-Coord),
we scaled the learning rate formula by a factor of $0.6/R$,
and for the global learning rate (Global) we scaled it by $0.2 / R$. We
estimate the diameter $D$ in the global learning rate formula online,
based on the number of attributes seen so far.  For CW, we found that the
parameters $\phi = 1.0$ and $a=1.0$ worked well in practice.

Table~\ref{tab:classification} presents average hinge loss and the
fraction of classification mistakes for each algorithm.  The Global and
Per-Coord algorithms are designed to minimize hinge loss,
and at this objective the Per-Coord algorithm consistently wins.
CW and PA are designed to maximize classification accuracy,
and on this objective Per-Coord and CW are the best algorithms.
The fact that the classification accuracy of Per-Coord is comparable to
that of a state-of-the-art binary classification algorithm is impressive
given the former algorithm's generality
(i.e., its applicability to arbitrary online
convex optimization problems such as online shortest paths).

\subsection {Large-scale logistic regression}

\begin{table}[t]
\caption{Additive regret incurred in the online setting, for logistic regression on various ads data sets.}
\label{tab:regression}
\vskip 0.15in
\begin{center}
\begin{small}
\begin{sc}
\begin{tabular}{lcccr}
\hline
\abovespace\belowspace
Data set & Global & Per-Coord \\
\hline
\abovespace
auto insurance    & 0.215 & {\bf 0.028 } \\

business cards    & 0.261 & {\bf 0.034 } \\

credit cards      & 0.225 & {\bf 0.029 } \\

credit report     & 0.148 & {\bf 0.012 } \\

forex             & 0.158 & {\bf 0.025 } \\

health insurance  & 0.232 & {\bf 0.032 } \\

life insurance    & 0.231 & {\bf 0.032 } \\

shoe              & 0.263 & {\bf 0.050 } \\

telefonica        & 0.171 & {\bf 0.026 } \\
\hline
\end{tabular}
\end{sc}
\end{small}
\end{center}
\vskip -0.2in
\end{table}
We collected data from a large search engine\footnote{No user-specific
data was used in these experiments.} consisting of random samples of
queries that contained a particular phrase, for example ``auto
insurance''.  Each data set has a few million examples.  We
transformed this data into an online logistic regression problem with
a feature vector $\theta_t$ for each ad impression, using features
based on the text of the ad and the query.  The target label $\ell_t$
is 1 if the ad was clicked, and -1 otherwise.  The loss function $f_t$
is the sum of the logistic loss, $\log \paren {1 + \exp(- \ell_t x_t
\theta_t) }$, and an L2 regularization term.

We compare gradient descent using the global learning rate from
\S\ref{subsec:better_global_rate} with gradient descent using the
per-coordinate rate given in \S\ref{subsec:per_coord_rate}.  We
scaled the formulas given in those sections by 0.1; this improved
performance for both algorithms but did not change the relative
comparison.  The feasible set was $[-1, 1]^n$.

Table~\ref{tab:regression} shows the regret incurred by the two
algorithms on various data sets.  Gradient descent with a
per-coordinate learning rate consistently obtains an order of
magnitude lower regret than with a global learning rate.
To
calculate regret, we computed the static optimal loss
$\min_{x \in F} \set {\sum_{t=1}^T f_t(x)}$
by running our per-coordinate
algorithm through the data many times until convergence.

\section{Related Work} \label{sec:related}

The use of different learning rates for different coordinates has been
investigated extensively in the neural network community.  There the
focus has been on empirical performance in the batch setting, and a
large number of algorithms have been developed; see for example
\cite{jacobs88}.  These algorithms are not designed to perform well in
an adversarial online setting, and for many of them it is
straightforward to construct examples where the algorithm incurs high
regret.

More recently, Hsu et al.\ \yrcite{hsu09} gave an algorithm
for choosing per-coordinate learning rates for gradient descent,
derive asymptotic rates of convergence in the batch setting, and
present a number of positive experimental results.

Confidence-weighted linear classification \cite{drezde08} 
and AROW \cite{crammer09} are similar
to our algorithm in that they make different-sized adjustments for
different coordinates, and in that common features are
updated less aggressively than rare ones.  Unlike our
algorithm, these algorithms apply only to classification problems and
not to general online convex optimization, and the guarantees are in
the form of mistake bounds rather than regret bounds.

In concurrent work \cite {\theotherpaper}, we generalize the results of this
paper to handle arbitrary feasible sets and a matrix (rather than a vector) 
of learning rate parameters.  Similar theoretical results were
obtained independently by Duchi et al.\ \yrcite{duchi10}.